\documentclass[letterpaper, 10 pt, journal, twoside]{IEEEtran}
\IEEEoverridecommandlockouts 
\usepackage[colorlinks=true,linkcolor=black,anchorcolor=black,citecolor=black,filecolor=black,menucolor=black,runcolor=black,urlcolor=black]{hyperref}
\usepackage{url}
\usepackage[skip=5pt,font=scriptsize]{caption}
\usepackage{graphicx}
\usepackage{subfigure}
\usepackage{etoolbox}
\usepackage{scalerel}
\usepackage{siunitx}
\usepackage{amsmath,textcomp}

\usepackage{amssymb}
\usepackage{amsthm}

\usepackage[ruled,vlined,linesnumbered]{algorithm2e}
\SetKwInput{KwInput}{Input}
\SetKwInput{KwOutput}{Output}
\SetKwInOut{KwInOut}{Parameter}
\usepackage{bbm}
\usepackage{booktabs}
\usepackage{url}
\usepackage{paralist}
\usepackage{color}
\usepackage{mathtools}
\usepackage{xspace}
\usepackage{verbatim}
\usepackage{epstopdf}
\usepackage{listings}
\usepackage{lstautogobble}
\usepackage[utf8]{inputenc}
\usepackage{calrsfs}
\usepackage{rotating,multirow}

\usepackage[utf8]{inputenc}
\let\labelindent\relax
\usepackage{enumitem}
\usepackage{hyperref}

\usepackage{titlesec}
\titlespacing\section{0pt}{3pt}{4pt}
\titlespacing\subsection{0pt}{2pt}{3pt}
\titlespacing\subsubsection{0pt}{1pt}{2pt}

\allowdisplaybreaks

\newtheorem{definition}{Definition}

\newtheorem{theorem}{Theorem}

\newtheorem{assumption}{Assumption}

\renewcommand{\footnotesize}{\fontsize{8pt}{11pt}\selectfont}

\DeclareMathAlphabet{\mathcal}{OMS}{cmsy}{m}{n}
\newcommand{\regtext}[1]{\mathrm{\textnormal{#1}}}

\newcommand{\ts}[1]{\textsuperscript{#1}}

\newcommand{\R}{{\mathbb{R}}}

\newcommand{\N}{{\mathbb{N}}}

\DeclarePairedDelimiter{\norm}{\lVert}{\rVert}

\DeclareMathOperator*{\argmin}{arg\,min}
\newcommand{\diag}[1]{\regtext{diag}\!\left(#1\right)}

\newcommand{\opt}{^\star}
\newcommand{\proj}{\regtext{proj}}

\newcommand{\vc}[1]{{\mathbf{#1}}}

\newcommand{\ones}{\vc{1}}

\newcommand{\zeros}{\vc{0}}

\newcommand{\zono}[1]{\mathcal{Z}\!\left(#1\right)}

\newcommand{\ctr}{\vc{c}}

\newcommand{\Gen}{\vc{G}}
\newcommand{\Acon}{\vc{A}}
\newcommand{\bcon}{\vc{b}}
\newcommand{\coef}{\vc{z}}
\newcommand{\lb}{\underline{\vc{l}}}
\newcommand{\ub}{\overline{\vc{l}}}

\newcommand{\idx}[1]{^{(#1)}}
\newcommand{\arridx}[2]{{\left(#1\right)}_{#2}}
\newcommand{\idxi}{\idx{i}}

\newcommand{\lbl}[1]{_{\regtext{#1}}}
\newcommand{\obs}{\lbl{obs}}
\newcommand{\goal}{\lbl{goal}}

\newcommand{\safe}{\lbl{safe}}

\newcommand{\total}{\lbl{total}}
\newcommand{\brk}{\lbl{brk}}

\def \tb{\textcolor{black}}
\newcommand{\tbold}[1]{#1}

\newcommand{\ngen}{{n\lbl{g}}}
\newcommand{\ncon}{{n\lbl{c}}}
\newcommand{\nrl}{{n\lbl{RL}}}

\newcommand{\niter}{{n\lbl{iter}}}

\newcommand{\nplan}{{n\lbl{plan}}}
\newcommand{\nbrk}{{n\brk}}
\newcommand{\ntraj}{{q}}
\newcommand{\nnoisegen}{{n_{\regtext{g},\noise}}}

\newcommand{\dyn}{\vc{f}}
\newcommand{\vcstate}{\vc{x}}
\newcommand{\action}{\vc{u}}
\newcommand{\plan}{\vc{p}}
\newcommand{\noise}{\vc{w}}
\newcommand{\statespace}{X}
\newcommand{\actionspace}{U}
\newcommand{\noisespace}{W}
\newcommand{\statedata}{\vc{X}}
\newcommand{\actiondata}{\vc{U}_-}

\newcommand{\data}{(\statedata_-,\statedata_+,\actiondata)}
\newcommand{\RLstate}{\hat{\vcstate}}
\newcommand{\RLaction}{\action}
\newcommand{\policy}{\pi}
\newcommand{\mimic}{\mu}

\newcommand{\rewfunc}{\rho}
\newcommand{\rew}{r}

\newcommand{\adjustfunc}[1]{\regtext{adjust}\!\left(#1\right)}

\newcommand{\reachset}{R}
\newcommand{\reachapprox}{{\hat{\reachset}}}
\newcommand{\reachfunc}[1]{\regtext{reach}\!\left(#1\right)}

\begin{document}

    
\title{Safe Reinforcement Learning Using Black-Box Reachability Analysis}

\author{Mahmoud Selim$^{1}$, Amr Alanwar$^{2}$, Shreyas Kousik$^{3}$, Grace Gao$^{3}$,
Marco Pavone$^{3}$, and Karl H. Johansson$^{4}$
\thanks{Manuscript received: Feb 24, 2022; Revised:
May 11, 2022; Accepted: June 12, 2022}
\thanks{This paper was recommended for publication by
Editor Jens Kober upon evaluation of the Associate Editor and Reviewers’
comments}
\thanks{This work was supported by the Swedish Research Council, and the Knut and Alice Wallenberg Foundation.
Toyota Research Institute provided funds to support this work.
The NASA University Leadership initiative (grant \#80NSSC20M0163) provided funds to assist the authors with their research, but this article solely reflects the opinions and conclusions of its authors and not any NASA entity.}
\thanks{$^{1}$Ain Shams University, Cairo, Egypt.
$^{2}$Jacobs University, Bremen, Germany.
$^{3}$Stanford University, Stanford, CA, USA.
$^{4}$KTH Royal Institute of Technology, Stockholm, Sweden.
Corresponding author: \texttt{mahmoud.selim@eng.asu.edu.eg}.}
\thanks{Digital Object Identifier (DOI): see top of this page}

}

\markboth{IEEE Robotics and Automation Letters. Preprint Version. Accepted June, 2022}
{Selim \MakeLowercase{\textit{et al.}}: BLACK-BOX REACHABILITY-BASED SAFETY} 

\maketitle

\begin{abstract}
Reinforcement learning (RL) is capable of sophisticated motion planning and control for robots in uncertain environments.
However, state-of-the-art deep RL approaches typically lack safety guarantees, especially when the robot and environment models are unknown.
To justify widespread deployment, robots must respect safety constraints without sacrificing performance. 
Thus, we propose a Black-box Reachability-based Safety Layer (BRSL) with three main components: (1)~data-driven reachability analysis for a black-box robot model, (2)~a trajectory rollout planner that predicts future actions and observations using an ensemble of neural networks trained online, and (3)~a differentiable polytope collision check between the reachable set and obstacles that enables correcting unsafe actions.
In simulation, BRSL outperforms other state-of-the-art safe RL methods on a Turtlebot 3, a quadrotor, a trajectory-tracking point mass, and a hexarotor in wind with an unsafe set adjacent to the area of highest reward.
\end{abstract}
\begin{IEEEkeywords}
Reinforcement Learning, Robot Safety, Task and Motion Planning
\end{IEEEkeywords}
\section{Introduction}
In reinforcement learning (RL), an agent perceives and reacts to consecutive states of its environment to maximize long-term cumulative expected reward \cite{book:sutton1998introduction}. 
One key challenge to the widespread deployment of RL in safety-critical systems is ensuring that an RL agent's policies are safe, especially when the system environment or dynamics are a black box and subject to noise \cite{mihatsch2002risk,garcia2015comprehensive}. 
In this work, we consider RL for guaranteed-safe navigation of mobile robots, such as autonomous cars or delivery drones, where safety means collision avoidance.
We leverage RL to plan complex action sequences in concert with data-driven reachability analysis to guarantee safety for a black-box system.

\begin{figure}[t]
    \centering
    \includegraphics[width=\columnwidth]{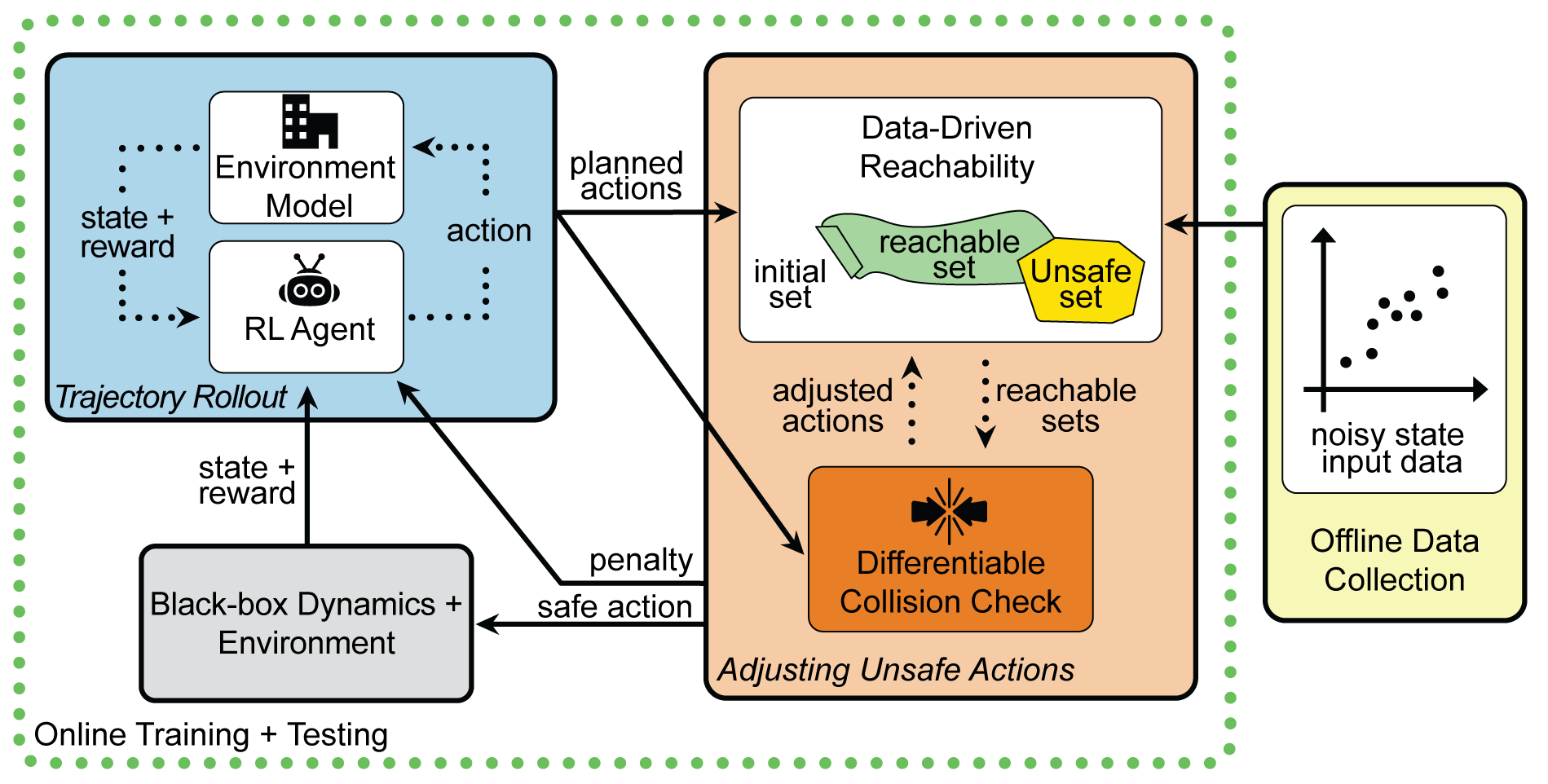}
    \vspace{-5mm}
    \caption{Overview of the proposed BRSL method (\href{https://youtube.com/playlist?list=PL7bkcpwNaUjz-S1b5KBzpgCZ1SP4DXsoi}{\textcolor{blue}{link to video}}).
    Given data collected offline (in yellow, right), we perform online safe training and deployment of an RL agent.
    The RL agent creates trajectory plans for a robot in a receding-horizon way as follows.
    Each planning iteration is one clockwise loop in the green dashed box.
    First (blue, top left), the agent predicts a possible future trajectory by rolling out its current policy with an ensemble of neural networks trained online to model the black-box environment (grey, bottom left).
    Second (orange, middle), the candidate plan is \emph{adjusted} to ensure safety using data-driven reachability and a constrained, differentiable method of collision-checking our robot's reachable sets.
    We execute a failsafe maneuver if the collision check is infeasible.
    Finally, the new safe plan is passed to the robot, and a penalty is passed to the RL agent for choosing unsafe action.}
    \label{fig:framework}
    \vspace{-6mm}
\end{figure}

\subsection{Related Work} \label{sec:related}
Safe RL aims to learn policies that maximize expected reward on a task while respecting safety constraints during both learning and deployment \cite{garcia2015comprehensive}. 
Existing methods can be roughly classified as \emph{objective-based} or \emph{exploration-based}, depending on how safety is formulated.
We first discuss these categories, then the specific case of mobile robot navigation, which we use to evaluate our proposed method.

Objective-based methods encourage safety by penalizing constraint violations in the objective.
This can be done by relating cumulative reward to the system's risk, such as the probability of visiting error states \cite{geibel2005risk}.
In practice, this results in an RL agent attempting to minimize an empirical risk measure (that is, an approximation of the probability of entering a dangerous or undesired state).
Similarly, one can penalize the probability of losing reward (by visiting an unsafe state) for a given action \cite{mihatsch2002risk}, in which case the agent minimizes temporal differences in the reward and thus also minimizes risk.
Another approach is to restrict policies to be ergodic with high probability, meaning any state can eventually be reached from any other state \cite{moldovan2012safe}.
This is a more general problem, which comes at a cost: feasible safe policies do not always exist, and the algorithms are far more complex. 
While these methods can make an agent prefer safe actions, they cannot guarantee safety during training or deployment.
Another group of objective-based algorithms aims to modify the Markov Decision Process (MDP) that the RL agent tries to optimize.
Some model safe optimization problems as maximizing an unknown expected reward function \cite{pmlr-v37-sui15}.
However, they exploit regularity assumptions on the function wherein similar decisions are associated with similar rewards.
They also assume the bandit setting, where decisions do not cause state transitions.
Others utilize constrained MDPs \cite{altman1999constrained} to enforce safety in various RL settings, either online or offline.
Online methods learn by coupling the iteration of numerical optimization algorithms (such as primal-dual gradient updates) with data collection \cite{BORKAR2005207,chow2017riskconstrained,tessler2018reward,bohez2019value}.
These algorithms have also been studied in exploration-based settings \cite{ding2020provably,efroni2020explorationexploitation}.
However, they provide no guarantees on safety during training.
On the other hand, offline schemes separate optimization and data collection \cite{achiam2017constrained,bharadhwaj2021conservative}.
They conservatively enforce safety constraints on every policy iteration but are more challenging to scale up.

Exploration-based methods modify the agent's exploration process instead of its optimization criterion.
Exploration is the process of learning about unexplored states by trying random actions or actions that are not expected to yield maximum reward (for example, an $\epsilon$-greedy strategy).
However, visiting unexplored states na\"ively can harm a robot or its environment.
To avoid this, one can aim to guarantee safety during both exploration and exploitation, in both training and testing, by modifying the exploration strategy to incorporate risk metrics \cite{riskmetricsRL}.
One can also use prior knowledge as an inductive bias for the exploration process \cite{garcia2015comprehensive,neuroevolutionaryinproceedings}; for example, one can provide a finite set of demonstrations as guidance on the task \cite{abbeel2005exploration}.
\tb{
Other approaches use control theory to guide or constrain the RL agent's actions.
Most of these approaches use system models with Lyapunov or control barrier functions (CBFs) to guarantee safety (or stability) of the system \cite{molnar2021model, liu2022robot, chow2018lyapunov}.
One can also combine data-driven approaches with model-free barrier or intervention functions \cite{cheng2019end, modelfreecbfsIVM, wagener2021safe}, or use robust CBFs to model uncertain parts of a system \cite{robustCBFarticle, cbfforunmodeleddynamics}.
Although these approaches can provide strong guarantees, most assume the system is control-affine, and need prior knowledge on some or all of the system model \cite{saferlforchanginglanes,conf:modelBasedReachRL,LewEtAl2021}, which may not always be feasible.
Finally, our method is most similar to \cite{dalal2018safe}, which learns a safety signal from data offline, then uses it to adjust an RL agent's controls at runtime.
This method uses a first-order safety approximation for fast adjustment, but (as we show) can be conservative.
}
 
\tb{
Safe navigation is a fundamental challenge in robotics, because robots typically have uncertain, nonlinear dynamics.
Classical techniques such as A$^*$ or RRT \cite[Ch. 5]{lavalle2006planning} have been proposed to solve the navigation problem without learning.
With these methods, safety has been enforced at different levels of the planning hierarchy, such as trajectory planning \cite{kousik2020bridging}, or low-level control \cite{leung2020infusing}.
More recently, however, learning-based methods have been proposed \cite{saferlforchanginglanes,conf:modelBasedReachRL,LewEtAl2021}.
Some safe RL navigation approaches depend on learning a value function of the expected time that an agent takes to reach the desired goal \cite{chen2016decentralized, chen2018socially}.
Other approaches depend on learning the actions of the robot in an end-to-end manner \cite{long2018optimally, e2efirst, tai2018socially}, meaning that the agent attempts to convert raw sensor inputs (e.g., camera or LIDAR) into actuator commands.
The key advantage of RL over traditional planners is in accelerating computation time and solution quality. 
}

\subsection{Proposed Method and Contributions}
\tb{We propose a Black-box Reachability-based Safety Layer (BRSL), illustrated in Fig.~\ref{fig:framework}, to enable strict safety guarantees for RL in an entirely data-driven way, addressing the above challenges of lacking robot and environment models \emph{a priori} and of enforcing safety for uncertain systems.
We note that the main advantage of BRSL is the use of data-driven methods to provide safety guarantees for black-box system dynamics.}
BRSL enforces safety by computing a system's forward reachable set, which is the union of all trajectories that the system can realize within a finite or infinite time when starting from a bounded set of initial states, subject to a set of possible input signals \cite{conf:reach1984}. 
Then, if the reachable set does not intersect with unsafe sets, the system is verified as safe, following similar arguments as in \cite{kousik2020bridging,leung2020infusing,althoffphdthesis}.

\textbf{Limitations.}
Our method requires an approximation for the upper bound of a system's Lipschitz constant, similar to \cite{LewEtAl2021,koller2018learning,alanwar2020data_conf}.
This results in a curse of dimensionality with respect to number of samples required to approximate the constant; note other sampling-based approaches scale similarly \cite{LewEtAl2021,conf:modelBasedReachRL}. 
Furthermore, we focus on a discrete-time setting, assume our robot can brake to a stop, and assume accurate perception of the robot's surroundings.
We leave continuous-time (which can be addressed with similar reachability methods to ours \cite{conf:modelBasedReachRL,althoffphdthesis}) and perception uncertainty to future work.


\textbf{Contributions.}
We show the following with BRSL:
\begin{enumerate}[leftmargin=15pt,noitemsep,partopsep=0pt,topsep=0pt,parsep=0pt]
\itemsep0em
    
    \item We propose a safety layer by integrating data-driven reachability analysis with a differentiable polytope collision check and a trajectory rollout planner.
    
    \item We demonstrate BRSL on robot navigation, where it outperforms a baseline RL agent, Reachability-based Trajectory Safeguard (RTS) \cite{conf:modelBasedReachRL}, Safe Advantage-based Intervention for Learning policies with Reinforcement (SAILR) \cite{wagener2021safe}, and Safe Exploration in Continuous Action Spaces (SECAS) \cite{dalal2018safe}.
    Our code is \href{https://github.com/Mahmoud-Selim/Safe-Reinforcement-Learning-for-Black-Box-Systems-Using-Reachability-Analysis}{\underline{\color{blue}{online}}}.
\end{enumerate}

Next, in Section \ref{sec:pb}, we provide preliminaries and formulate our safe RL problem. 
Sections \ref{sec:solu} and \ref{sec:eval} discuss and evaluate the proposed approach.
Finally, Section \ref{sec:con} presents concluding remarks and discusses future work.




\section{Preliminaries and Problem Formulation} \label{sec:pb}

This section presents the notation, set representations, system dynamics, and reachable set definitions used in this work. We then pose our safe RL problem.

\subsection{Notation and Set Representations}\label{subsec:notation}
The $n$-dimensional real numbers are $\R^n$, the natural numbers are $\N$, and the integers from $n$ to $m$ are $n{:}m$.
We denote the element at row $i$ and column $j$ of matrix $\vc{A}$ by 
$(\vc{A})_{i,j}$, column $j$ of $\vc{A}$ by $\arridx{\vc{A}}{:\,,j}$, and the element $i$ of vector $\vc{a}$ by $\arridx{\vc{a}}{i}$.
An $n\times m$ matrix of ones is $\ones_{n\times m}$.
For $\vc{A} \in \R^{n\times m}$ and $\vc{x} \in \R^n$, we use the shorthand $\vc{A} - \vc{x} = \vc{A} - \vc{x}\ones_{1\times m}$.
The $\diag{\cdot}$ operator places its arguments block-diagonally in a matrix of zeros.
For a pair of sets $A$ and $B$, the Minkowski sum is $A + B = \{\vc{a}+\vc{b}\ |\ \vc{a} \in A, \vc{b} \in B\}$, and the Cartesian product is $A \times B = \{(\vc{a},\vc{b})\ |\ \vc{a} \in A, \vc{b} \in B\}$.

We represent sets using constrained zonotopes, zonotopes, and intervals, because they enable efficient Minkowski sum computation (a key part of reachability analysis) \cite{althoffphdthesis} and collision checking via linear programming (critical to safe motion planning) \cite{conf:const_zono}.
A \emph{constrained zonotope} \cite{conf:const_zono} is a convex set parameterized by a center $\ctr \in \R^n$, generator matrix $\Gen \in \R^{n \times \ngen}$, constraint matrix $\Acon \in $ $\R^{\ncon \times \ngen}$, and constraint vector $\bcon \in \R^{\ncon}$ as
\begin{equation}\label{eq:conszono}
    \zono{\ctr,\Gen,\Acon,\bcon} = \big\{
        \ctr + \Gen\coef\ |\ \Acon\coef = \bcon,\ \norm{\coef}_\infty \leq 1
    \big\}.
\end{equation}
By \cite[Thm. 1]{conf:const_zono}, every convex, compact polytope is a constrained zonotope and vice-versa.
For polytopes represented as an intersection of halfplanes, we convert them to constrained zonotopes by finding a bounding box, then applying the halfspace intersection property in \cite{raghuraman2020set_ops_conzono}.

A zonotope is a special case of a constrained zonotope without equality constraints (but with $\norm{\coef}_\infty \leq 1$), which we denote $\zono{\ctr,\Gen}$.
For $Z = \zono{\ctr,\Gen} \subset \R^n$ and a linear map $L$, we have $LZ  = \zono{L\ctr, L\Gen}$; we denote $-Z = -1Z$.
The Minkowski sum of two zonotopes $Z_1 = \zono{\ctr_1,\Gen_1}$ and $Z_2 = \zono{\ctr_2,\Gen_2}$ is given by $Z_1 + Z_2 = \zono{\ctr_1+\ctr_2,[\Gen_1,\Gen_2]}$ \cite{althoffphdthesis}.
For an $n$-dimensional interval with lower (resp. upper) bounds $\lb \in \R^n$ (resp. $\ub$), we abuse notation to represent it as a zonotope $Z = \zono{\lb,\ub} \subset \R^n$, with center $\tfrac{1}{2}(\lb+\ub)$ and generator matrix $\diag{\tfrac{1}{2}(\ub - \lb)}$.

\subsection{Robot and Environment}
We assume the robot can be described as a discrete-time, nonlinear control system with state $\vcstate_{k} \in \statespace \subset \R^n$ at time $k \in \N$.
We assume the state space $\statespace$ is compact.
The input $\action_{k}$ is drawn from a zonotope $\actionspace_k \subseteq \actionspace$ at each time $k$, where $\actionspace \subset \R^m$ is a zonotope of all possible actions.
We denote process noise by $\noise_{k} \in \noisespace \subset \R^n$, where $\noisespace$ is specified later in Assumption \ref{assumption:noise_zonotope}.
Finally, we denote the black box (i.e., unknown) dynamics $\dyn: \statespace\times\actionspace\times\noisespace \to \statespace$, for which
\begin{align}\label{eq:sys}
        \vcstate_{k+1} &=  \dyn(\vcstate_{k}, \action_{k}) + \noise_{k}.
\end{align}
We further assume that $\dyn$ is twice differentiable and Lipschitz continuous, meaning there exists a \emph{Lipschitz constant} $L^\star$ such that, if $\forall$ $\vc{z}_1, \vc{z}_2 \in \R^{n+m}$ with $\vc{z}_j= (\vcstate_j,\action_j)$, then $\norm{\dyn(\vc{z}_1) - \dyn(\vc{z}_2)} \leq L^\star \norm{\vcstate_1 - \vcstate_2}$.
We denote the initial state of the system as $\vcstate_{0}$, drawn from a compact set $\statespace_{0} \subset \R^n$.
Note that this formulation leads to an MDP.

To enable safety guarantees, we leverage the notion of failsafe maneuvers from mobile robotics \cite{kousik2020bridging,magdici2016fail}. 
\begin{assumption}\label{assumption:stop}
We assume the dynamics $\dyn$ are invariant to translation in position, and the robot can brake to a stop in $\nbrk \in \N$ time steps and stay stopped indefinitely.
That is, there exists $\action\brk \in \actionspace$ such that, if the robot is stopped at state $\vcstate_k$, and if $\vcstate_{k+1} = \dyn(\vcstate_k,\action\brk)$, then $\vcstate_{k+1} = \vcstate_k$.
\end{assumption}
\noindent Note, many real robots have a braking safety controller available, similar to the notion of an invariant set \cite{LewEtAl2021,conf:modelBasedReachRL}.
Also, failsafe maneuvers exist even when a robot cannot remain stationary, such loiter circles for aircraft \cite{fridovich2019safely,althoff2015online}.

We require that process noise obeys the following assumption for numerical tractability and robustness guarantees.

\begin{assumption}\label{assumption:noise_zonotope}
Each $\noise_{k}$ is drawn uniformly from a \emph{noise zonotope} $\noisespace = \zono{\ctr_\noise,\Gen_\noise}$ with $\nnoisegen$ generators.
\end{assumption}
\noindent This formulation does not handle discontinuous changes in noise.
However, there exist zonotope-based techniques to identify a change in $\noisespace$ \cite{shetty2020predicting}, after which one can compute the system's reachable set as in the present work.
We leave measurement noise and perception uncertainty to future work.
We also note, in the case of Gaussian or unbounded noise, one can overapproximate a confidence level set of a probability distribution using a zonotope \cite{althoffphdthesis,shetty2020predicting}.

We denote unsafe regions of state space, or \emph{obstacles}, as $\statespace\obs \subset \statespace$.
We assume obstacles are static but different in each episode, as the focus of this work is not on predicting other agents' motion.
Furthermore, reachability-based frameworks exist to handle other agents' motion \cite{leung2020infusing,vaskov2019not}, so the present work can extend to dynamic environments.

We further assume the robot can instantaneously sense all obstacles (that is, $\statespace\obs$) and represent them as a union of constrained zonotopes.
In the case of sensing limits, one can determine a minimum distance within which obstacles must be detected to ensure safety, given a robot's maximum speed and braking distance \cite[Section 5]{kousik2020bridging}.

\subsection{Reachable Sets}\label{subsec:reachable_sets_defn}
We ensure safety by computing our robot's forward reachable set (FRS) for a given motion plan, then adjusting the plan so that the FRS lies outside of obstacles.
We define the FRS, henceforth called the reachable set, as follows:
\begin{definition} 
The reachable set $\reachset_{k}$ at time step $k$, subject to a sequence of inputs $\action_{j} \in \actionspace_{j} \subset \mathbb{R}^m$, noise $\noise_{j} \in \noisespace$ $\forall\ j \in \{ 0, \dots, k-1\}$, and initial set $\statespace_{0} \in  \mathbb{R}^n$, is the set
\begin{align}\begin{split}\label{eq:reachable_set_k}
        \reachset_{k} = \big\{&\vcstate_{k} \in \mathbb{R}^n \, \big|\ 
            \vcstate_{j+1} = \dyn(\vcstate_{j}, \action_{j}) + \noise_{j},\ \vcstate_{0} \in \statespace_{0},\\ &\action_{j} \in \actionspace_{j},\ \regtext{and}\ \noise_{j} \in \noisespace,\  \forall\ j = 0,\cdots,k-1\big\}.
\end{split}\end{align}
\end{definition}

Recall that we treat the dynamics $\dyn$ as a black box (e.g., a simulator), which could be nonlinear and difficult to model, but we still seek to conservatively approximate (that is, overapproximate) the reachable set $\reachset_k$.


\subsection{Safe RL Problem Formulation}
We denote the state of the RL agent at time $k$ by  $\RLstate_{k} \in \R^\nrl$, which contains the state $\vcstate_k$ of the robot plus information such as sensor measurements and previous actions.
At each time $k$, the RL agent chooses $\RLaction_{k}$.
Recall that $\statespace\obs \subset \R^n$ denotes obstacles.
For a given task, we construct a reward function $\rewfunc: (\RLstate_{k},\action_k) \mapsto \rew_{k} \in \R$ (examples of $\rewfunc$ are given in Section \ref{sec:eval}).
At time $k$, let $\plan_k = (\action_j)_{j=k}^\nplan$ denote a \emph{plan}, or sequence of actions, of duration $\nplan \in \N$.

Then, our safe RL problem is as follows.
We seek to learn a policy $\policy_\theta: \RLstate_{k} \mapsto \action_{k}$, represented by a neural network with parameters $\theta$, that maximizes expected cumulative reward.
Note that the policy can be deterministic or stochastic.
Since rolling out the policy na\"ively may lead to collisions, we also seek to create a safety layer between the policy and the robot (that is, to ensure $\reachset_j \cap \statespace\obs = \emptyset$ for all $j \geq k$).


\setlength{\textfloatsep}{0.45cm}
\setlength{\floatsep}{0.45cm}
\begin{algorithm}[t]
\caption{Safe RL with BRSL}
\label{alg:rlwithbrsl}
\textbf{initialize} the RL agent with a random policy $\policy_\theta$, environment model $\mimic_\phi$, empty replay buffer $B$, max number of time steps $\niter$, and a safe plan $\plan_0$

\For{each episode}{
    \textbf{initialize} task with reward function $\rewfunc$
    
    $\RLstate_1 \leftarrow $ observe initial environment state \label{ln:observe1}
    
    \For{$k = 1:\niter$}{
        
    \tbold{ $\plan_k \leftarrow$ roll out a trajectory}
     
         $\reachapprox_k \leftarrow \zono{\vcstate_k,\zeros}$ // init. reachable set \label{ln:initZono}
        
         $\left(\reachapprox_j\right)_{j=k}^{k+\nplan} \leftarrow \reachfunc{\reachapprox_k,\plan_k}$ // use Alg. \ref{alg:LipReachability}  \label{ln:reachAlg1}
        
        \If{any $\reachapprox_j \cap \statespace\obs \neq \emptyset$ \label{ln:if_unsafe}}{
            
            \textbf{try} $\plan_k \leftarrow \adjustfunc{\plan_k,\statespace\obs}$ // use Alg. \ref{alg:adjust}
            
            \textbf{catch} execute failsafe maneuver; continue\label{ln:catch_unsafe} 
        }
        
        $\action_k \leftarrow$
        get first (safe) action from $\plan_k$ \label{ln:get_safe_action}

         $r_k \leftarrow \rewfunc(\RLstate_k,\action_k)$ // get reward \label{ln:getReward}
        
         $\RLstate_{k + 1} \leftarrow $ observe next environment state \label{ln:observe2}
        
         \textbf{add} $(\RLstate_k, \action_k, r_k, \RLstate_{k + 1})$ to $B$ \label{ln:addToBuffer}
        
         \textbf{train} the RL agent $\policy_\theta$ and the environment model $\mu_\phi$ using minibatch from $B$ \label{ln:train_rl_agent}
    }
}
\end{algorithm}

\section{Black-box Reachability-based Safety Layer}\label{sec:solu}

\tbold{We unite three components into our BRSL system for collision-free motion planning without a dynamic model of the robot or its surroundings \textit{a priori}.
The first component is an environment model, learned online.
To find high reward actions (i.e., motion plans) for this model, the second component is an RL agent.
Since the agent may create unsafe plans, our third component is a safety layer that combines data-driven reachability analysis with differentiable collision checking to enable safe trajectory optimization.
Theorem \ref{thm:BRSL_is_safe} summarizes safety via BRSL.}


BRSL is summarized in Algorithm \ref{alg:rlwithbrsl}.
It uses a receding-horizon strategy to create a new safe plan $\plan_k$ in each $k$\ts{th} receding-horizon motion planning iteration.
Consider a single planning iteration (that is, time step $k$) (Lines \ref{ln:observe1}--\ref{ln:train_rl_agent}).
Suppose the RL agent has previously created a safe plan $\plan_{k-1}$ (such as staying stopped indefinitely).
At the beginning of the iteration, BRSL creates a new plan $\plan_k$ by rolling out the RL agent along with an environment model. 
Next, BRSL chooses a safe action by adjusting the \tbold{rolled-out} action sequence (Lines \ref{ln:if_unsafe}--\ref{ln:catch_unsafe}) such that the corresponding reachable set (computed with Algorithm \ref{alg:LipReachability}) is collision-free and ends with a failsafe maneuver. 
If the adjustment procedure (as in Algorithm \ref{alg:adjust}) fails to find a safe plan, then the robot executes the  failsafe maneuver.
Finally, BRSL sends the first action in the current safe plan to the robot, gets a reward, and trains the RL agent and environment model (Lines \ref{ln:get_safe_action}--\ref{ln:train_rl_agent}). 
To enable training our environment model online, we collect data in a replay buffer $B$ at each time $k$ (Line \ref{ln:addToBuffer}).
We note that BRSL can be used during both training and deployment.
That is, the safety layer can operate even for an untrained policy.
Thus, for training, we initialize $\policy_\theta$ with random weights.

\tbold{To proceed, we detail our methods for data-driven reachability and adjusting unsafe actions.}

 \subsection{Data-Driven Reachability Analysis}\label{subsec:data_driven_reachability_analysis}
 
\begin{algorithm}[t]
\caption{Black-box System Reachability \cite{alanwar2020data_conf}}\label{alg:LipReachability}
\KwInput{initial reachable set $\hat{\reachset}_0$, actions $(\action_j)_{j=k}^{k+\nplan}$} 

\KwInOut{state/action data $\data$, noise zonotope $\noisespace = \zono{\ctr_\noise,\Gen_\noise}$, Lipschitz constant $L^\star$, covering radius $\delta$}
    

\tbold{
$Z_\epsilon \leftarrow \zono{\zeros,\diag{\arridx{\vc{L^\star}}{1} \arridx{\vc{\delta}}{1}/2,\cdots,\arridx{\vc{L^\star}}{n} \arridx{\vc{\delta}}{n}/2}}$ \label{ln:alglipZeps}
}
\For{$j = k:(k+\nplan)$}{      
    
    \tbold{$\vc{M}_j \leftarrow \left(\statedata_+ - \ctr_\noise\right)
        \begin{bmatrix} 
            \ones_{1 \times t\total} \\
            \statedata_- - \vcstate^\star_j \\
            \actiondata - \action_j
        \end{bmatrix}^\dagger$ \label{ln:alglipMtilde}}
    
    
    $\lb  \leftarrow \min_j \Bigg( {(\statedata_{+})}_{:\,,j} - \vc{M}_j
        \begin{bmatrix}
            1 \\
            {(\statedata_-)}_{:\,,j} -  \vcstate^\star_j \\
            {(\actiondata)}_{:\,,j} - \action_j
    \end{bmatrix} \Bigg)$
    
    $\ub \leftarrow$ same as $\lb$, but use max instead of min  \label{ln:alglipupper}
    
    
    
     $Z_{L} \leftarrow \zono{\lb,\ub} - \noisespace$ and $\actionspace_j \leftarrow \zono{\action_j,\zeros}$ \label{ln:alglipZl}
     
     $\reachapprox_{j+1} {\leftarrow} \vc{M}_j (\ones \times (\reachapprox_j-\vcstate^\star_j)  \times (U_j- \action_j)) +  \noisespace +  Z_L + Z_\epsilon$. \label{ln:alglipRhat}
}
\textbf{return} $(\reachapprox_j)_{j=k}^{k+\nplan}$ // overapproximates \eqref{eq:reachable_set_k}
\end{algorithm}

BRSL performs data-driven reachability analysis of a plan $\plan_k = (\action_j)_{j=k}^{\nplan}$ using Algorithm \ref{alg:LipReachability}, based on \cite{alanwar2020data_conf}.
\tbold{Algorithm \ref{alg:LipReachability} overapproximates the reachable set as in \eqref{eq:reachable_set_k} by computing a zonotope $\reachapprox_j \supseteq \reachset_j$ for each time step of the current plan.}

Our reachability analysis uses noisy trajectory data of the black-box system model collected offline; we use data collected online only for training the policy and environment model.
We consider $\ntraj$ input-state trajectories of lengths $t_i \in \N$, $i = 1,\cdots,\ntraj$, with total duration $t\total = \sum_i^\ntraj t_i$.
\tbold{We denote the data as $(\vcstate\idxi_k)_{k=0}^{t_i}$, $(\action\idxi_k)_{k=0}^{t_i-1}$, $i=1, \cdots, \ntraj$.
To ease notation for the various matrix operations needed in Algorithm \ref{alg:LipReachability}, we collect the data in matrices:
\begin{subequations}\label{eq:data}
\begin{align}
    \statedata_- &= \left[\vcstate\idx{1}_0,\cdots,\vcstate\idx{1}_{t_1-1}, \vcstate\idx{2}_0, \cdots, \vcstate\idx{\ntraj}_0, \cdots,\vcstate\idx{\ntraj}_{t_\ntraj-1} \right],\\
    \statedata_+ &= \left[\vcstate\idx{1}_1,\cdots,\vcstate\idx{1}_{t_1}, \vcstate\idx{2}_1, \cdots, \vcstate\idx{\ntraj}_1, \cdots,\vcstate\idx{\ntraj}_{t_\ntraj} \right],\\
    \actiondata &= \left[\action\idx{1}_0,\cdots,\action\idx{1}_{t_1-1}, \action\idx{2}_0,\cdots,\action\idx{\ntraj}_0,\cdots,\action\idx{\ntraj}_{t_\ntraj-1} \right].\label{eq:input_data}
\end{align}
\end{subequations}
Note, the time steps are different in $\statedata_-$ and $\statedata_+$ to simplify considering state transitions corresponding to the actions in $\actiondata$.
}
Selecting enough data to sufficiently capture system behavior is a challenge that depends on the system, though specific sampling strategies exist for some systems \cite{conf:modelBasedReachRL}.

We must approximate the Lipschitz constant of the dynamics for our reachability analysis, which we do from the data $\data$ with the method in \cite[Section 4, Remark 1]{alanwar2020data_conf}.
\tbold{We also require a data covering radius $\delta$ such that, for any data point $\vc{z}_1 \in \statespace\times\actionspace$, there exists another data point $\vc{z}_2 \in \statespace\times\actionspace$ for which $\norm{\vc{z}_1 - \vc{z}_2}_2 \leq \delta$.}
We assume sufficiently many data points are known \emph{a priori} to upper-bound $L^\star$ and lower-bound $\delta$; and, we assume $L^\star$ and $\delta$ are the same for offline data collection and online operation. 
Note, prior work assumes similar bounds \cite{koller2018learning,alanwar2020data_conf}.

\tbold{We find that the reachable set becomes conservative (i.e., large) if the same $L^\star$ and  $\delta$ are used for every dimension, because the true dynamics are typically scaled differently in each state dimension.
To mitigate this source of conservativeness, we approximate a different $(L^\star)_i$ and $(\delta)_i$ for each dimension, which we then use to compute a Lipschitz zonotope $Z_\epsilon$ (see Line \ref{ln:alglipZeps} of Algorithm \ref{alg:LipReachability}).
Note, this is an improvement over prior work \cite{alanwar2020data_conf}.
}

\subsection{Adjusting Unsafe Actions}

\tbold{After the RL agent \tbold{rolls out} a plan $\plan_k$, the safety layer adjusts it to ensure it is safe.
This is done by checking the intersection of the plan's reachable sets with unsafe sets.}
Note, our proposed adjustment procedure does not depend on $\policy_\theta$, only on the unsafe sets around the robot.
The plan is applied to the environment if all of its actions are safe; otherwise, we search for a safe plan.
One strategy for finding a safe plan is to sample randomly in the action space \cite{conf:modelBasedReachRL}, but this can be prohibitively expensive in the large action spaces that arise from choosing control inputs at multiple time steps.
Instead, we use gradient descent to adjust  our plan such that the reachable sets are not in collision, and such that the plan has a failsafe maneuver.

We adjust unsafe actions using Algorithm \ref{alg:adjust}.
If the algorithm does not complete within the duration of one time step (in other words, we fix the rate of receding-horizon planning), we terminate it and continue our previously-found safe plan.
Our method steps through each action in a plan $\plan$ and performs the following.
First, we compute the reachable set for all remaining time steps with Algorithm \ref{alg:LipReachability} (Line \ref{ln:adjust:reach}).
Second, we collision check the reachable set (Line \ref{ln:adjust:collision_check}) as detailed below.
Third, if the reachable sets are in a collision, we compute the gradient of the collision check and perform projected gradient descent (Line \ref{ln:adjust:get_gradient_and_step}) as in fig \ref{fig:zonotope_intersection}.
Finally, if the algorithm converges to a safe plan, we return it, or else return ``unsafe.''
Note, the final plan must have a failsafe maneuver (Line \ref{ln:adjust:check_if_safe}).

We collision check reachable and unsafe sets, all represented as constrained zonotopes, as follows.
Consider two constrained zonotopes, $Z_1 = \zono{\ctr_1,\Gen_1,\Acon_1,\bcon_1}$ and $Z_2 = \zono{\ctr_2,\Gen_2,\Acon_2,\bcon_2}$.
Applying \cite[Prop. 1]{conf:const_zono}, their intersection is $Z_\cap = Z_1 \cap Z_2 = \zono{\ctr_{\cap},\Gen_{\cap},\Acon_{\cap},\bcon_{\cap}}$, given by
\begin{align}
Z_\cap =  \zono{\ctr_1, [\Gen_1, \zeros], \begin{bmatrix}
            \Acon_1 & \zeros \\
            \zeros & \Acon_2 \\
            \Gen_1 & -\Gen_2
        \end{bmatrix}, \begin{bmatrix}
            \bcon_1 \\ \bcon_2 \\ \ctr_2 - \ctr_1
        \end{bmatrix}}.
\end{align}
We check if $Z_1 \cap Z_2$ is empty by solving a linear program, as per \cite[Prop. 2]{conf:const_zono}:
\begin{align}\label{prog:collision_check}
    v\opt = \min_{\coef,v} \left\{v\ |\ \Acon_{\cap} \coef = \bcon_{\cap} \ \regtext{and}\ |\coef| \leq v\right\},
\end{align}
with $|\coef|$ taken elementwise; $Z_\cap$ is nonempty iff $v \leq 1$.
Note, \eqref{prog:collision_check} is feasible when $Z_1$ and $Z_2$ have feasible constraints.


\begin{figure}
\centering     
\subfigure[]{\label{fig:Ng1}\includegraphics[width=0.4\columnwidth]{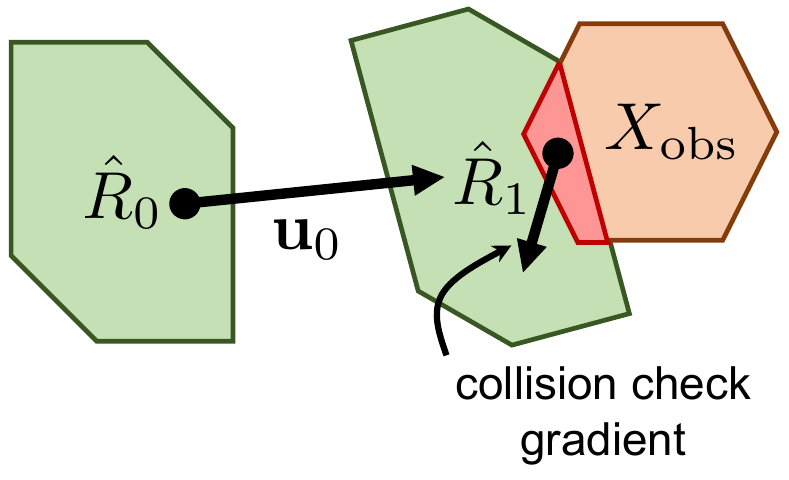}}
\hspace{5mm}
\subfigure[]{\label{fig:Ng2}\includegraphics[width=0.4\columnwidth]{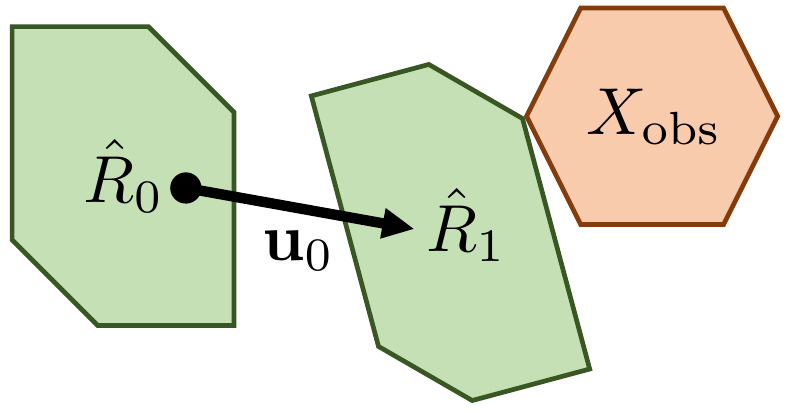}}
\caption{We move zonotope reachable sets out of intersection (see Alg. \ref{alg:adjust}) by using the gradient of a collision check, shown in (a), to adjust $\action_0$ so the reachable sets are out of collision as shown in (b).} 
\label{fig:zonotope_intersection}
\vspace{-3mm}
\end{figure}

We use gradient descent to move our reachable sets $\reachapprox_k$ out of collision.
Since we use \eqref{prog:collision_check} for collision checking, we differentiate its solution with respect to the problem parameters using \cite{amos2017optnet,chung2021constrained}.
Let $\hat{\ctr}_k$ denote the center of $\reachapprox_k$.
Per Algorithm \ref{alg:LipReachability}, $\reachapprox_k$ is a function of $\action_0,\cdots,\action_{k-1}$.
Let $(\coef\opt,v\opt)$ be an optimal solution to \eqref{prog:collision_check} when the problem parameters \tbold{(i.e., the input constrained zonotopes)} are $\reachapprox_k$ and an unsafe set.
Collision avoidance requires $v\opt > 1$ \cite[Prop. 2]{conf:const_zono}.
We compute the gradient $\nabla_{\RLaction_k} v\opt$ with respect to the input action (assuming a constant linearization point) using a chain rule recursion with $i = 0,\cdots,\nplan$ given by
\begin{align}\label{eq:collision_chain_rule}
    \nabla_{\RLaction_{h}} v\opt = \nabla_{\hat{\ctr}_k} v\opt
        \nabla_{\hat{\ctr}_{k - 1}} \hat{\ctr}_k
        \left(\prod_{j=h + 2}^{j = k - 1} \nabla_{\hat{\ctr}_{j - 1}} \hat{\ctr}_j\right) 
         \nabla_{\RLaction_{h}}  \hat{\ctr}_{h + 1},
\end{align}
with $h = k - i$.
The gradients of $\hat{\ctr}_k$ are given by
\begin{subequations}
\begin{align}
    &\nabla_{\hat{\ctr}_{k - 1}} \hat{\ctr}_k = (\vc{M}_{k - 1})_{(1:1+n),(1:1+n)},\ \regtext{and}
    \label{eq:collision_chain_rule_for_input_zono1} \\
    &\nabla_{u_{k - 1}} \hat{\ctr}_k = (\vc{M}_{k-1})_{:,(n+1:n+1+m)},
    \label{eq:collision_chain_rule_for_input_zono2}
\end{align}
\end{subequations}
where $\vc{M}_{k-1}$ is computed as in Algorithm \ref{alg:LipReachability}, Line \ref{ln:alglipMtilde}, and $n$ and $m$ are the state and action dimensions. 
After using  $\nabla_{\action_k}v\opt$ for gradient descent on $\action_k$, we project $\action_k$ to the set of feasible controls: $\proj_{\actionspace_k}(\action_k) = \argmin_{\vc{v} \in \actionspace_k} \left\{\norm{\action_k - \vc{v}}_2^2\right\}$.
The resulting controls may be unsafe, so we collision-check the final reachable sets at the end of Algorithm \ref{alg:adjust}.

\begin{algorithm}[t]
\caption{Adjusting Unsafe Actions}
\label{alg:adjust}

\KwInput{plan $\plan_k = (\action_j)_{j=k}^{\nplan}$, obstacles $\statespace\obs$, initial reachable set $\reachapprox_k$, step size $\gamma$, time limit $t\lbl{max}$, time steps required to stop $n\brk$} 
// note $\plan_k$ has failsafe $\action\brk$ for all $j > k+\nplan$

$\plan\safe \leftarrow \plan_k$ // initialize with given plan

\For{$j = k:(k+\nplan + n\brk)$\label{ln:adjust:for_loop}}{
    \While{time limit not exceeded\label{ln:adjust:while_loop}}{
        $(\reachapprox_j)_{j=k}^{k+\nplan} \leftarrow \reachfunc{\reachapprox_j,\plan\safe}$ // use Alg. \ref{alg:LipReachability} \label{ln:adjust:reach}
    
        $v\opt \leftarrow $ collision check $\reachapprox_j \cap \statespace\obs$ using \eqref{prog:collision_check} \label{ln:adjust:collision_check}
        
        \uIf{$v\opt \leq 1$ (i.e., in collision)}{        
            $\action_j \leftarrow \proj_{\actionspace_j}\left(\action_j + \gamma \nabla_{\action_j} v\opt\right)$ // using \eqref{eq:collision_chain_rule} \label{ln:adjust:get_gradient_and_step}
            
            
        }
        \Else{
            \textbf{break} and restart inner while loop \label{ln:adjust:break_while_loop}
        }
    }
}

\eIf{all $\reachapprox_j \cap \statespace\obs = \emptyset$ and $\vcstate_n$ is stopped\label{ln:adjust:check_if_safe}}{
    
    \textbf{return} $\plan\safe = (\action_j)_{j=k}^{\nplan}$ // found new safe plan
}{
    \textbf{return} error ``unsafe'' // failed to find safe plan\label{ln:adjust:return_unsafe}
}
\end{algorithm}

\begin{figure*}[t]
\centering     
\begin{subfigure}[h][\tbold{Turtlebot3 Environment}]{\label{fig:tb_env}\includegraphics[scale=0.1]{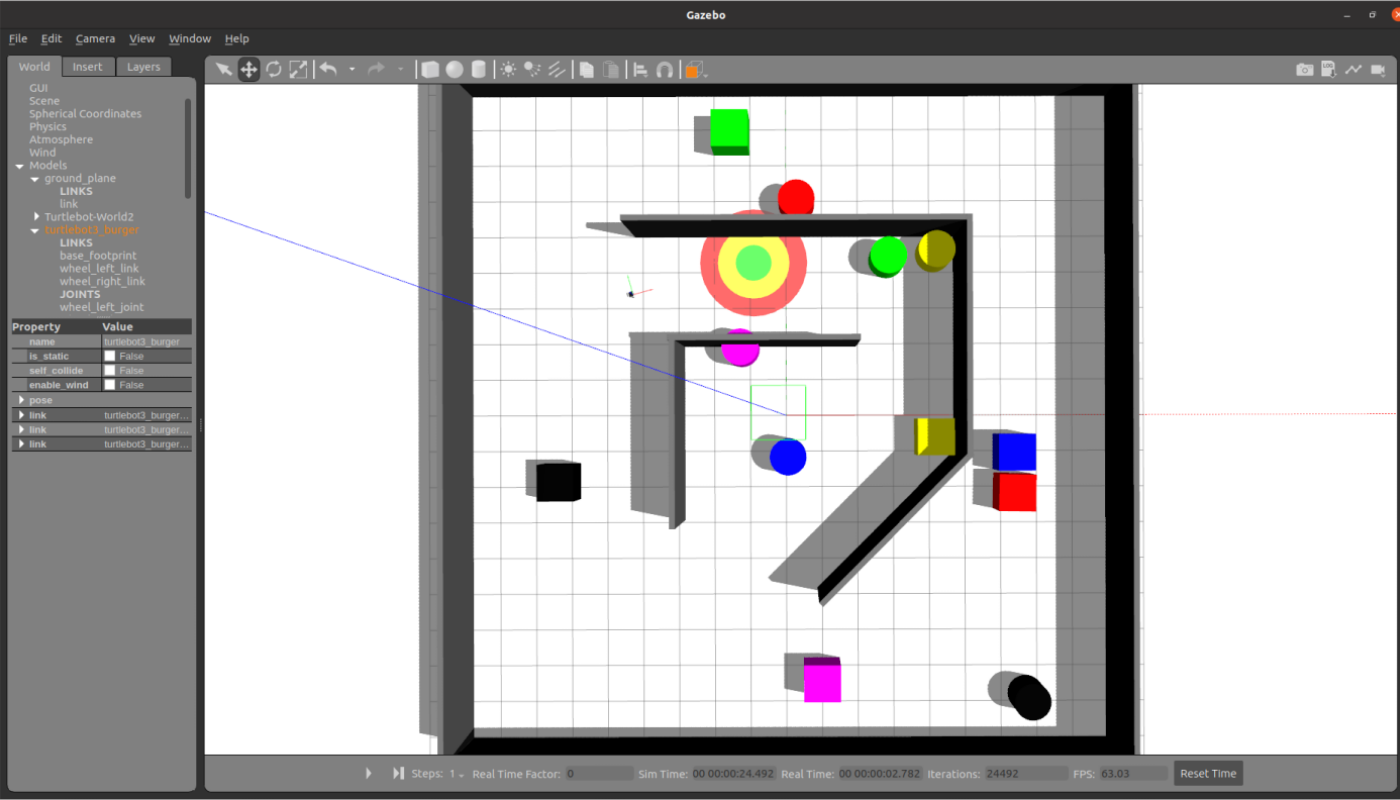}}
\end{subfigure}
\begin{subfigure}[h][\tbold{Quadrotor Environment}]{\label{fig:quad_env}\includegraphics[scale=0.1]{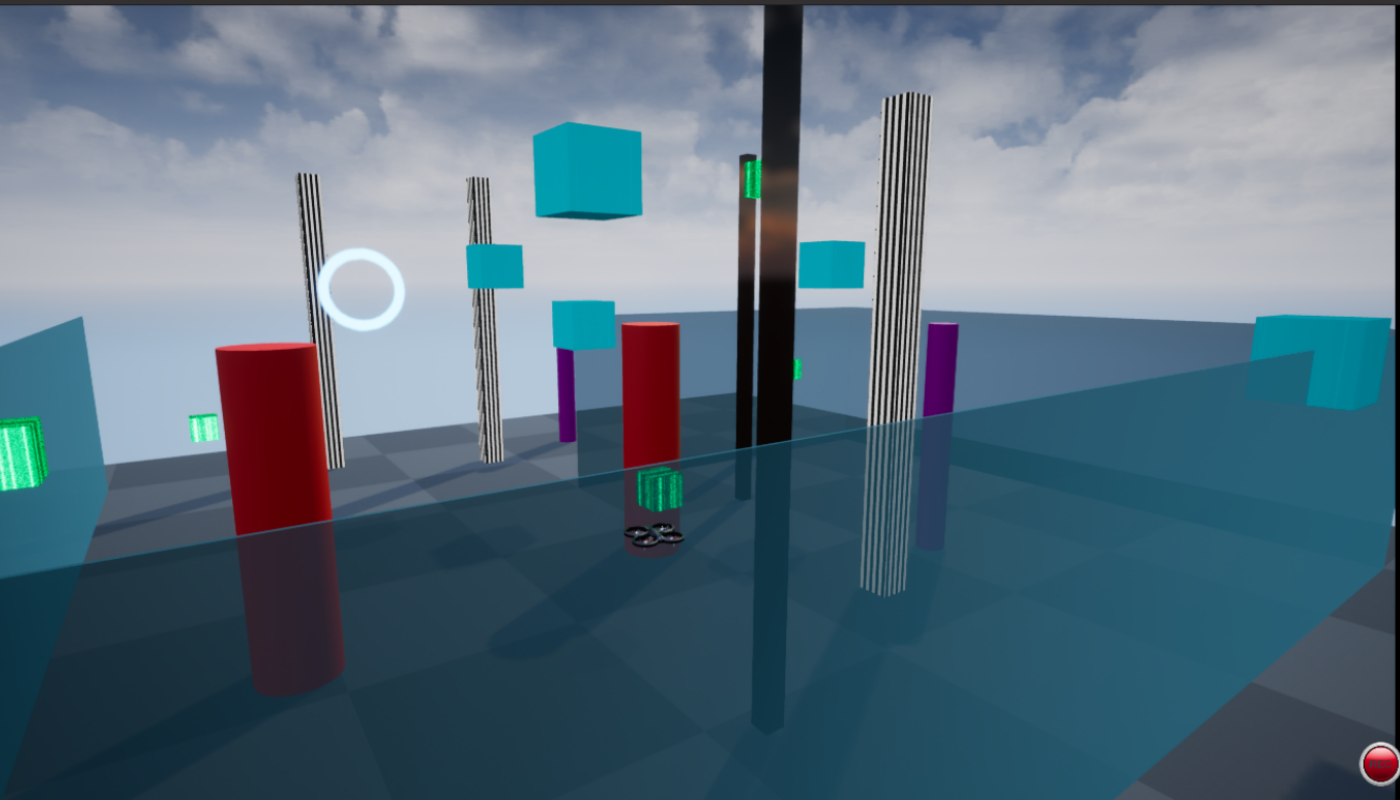}}
\end{subfigure}
\begin{subfigure}[h][\tbold{Hexarotor Environment}]{\label{fig:hex_env}\includegraphics[scale=0.1]{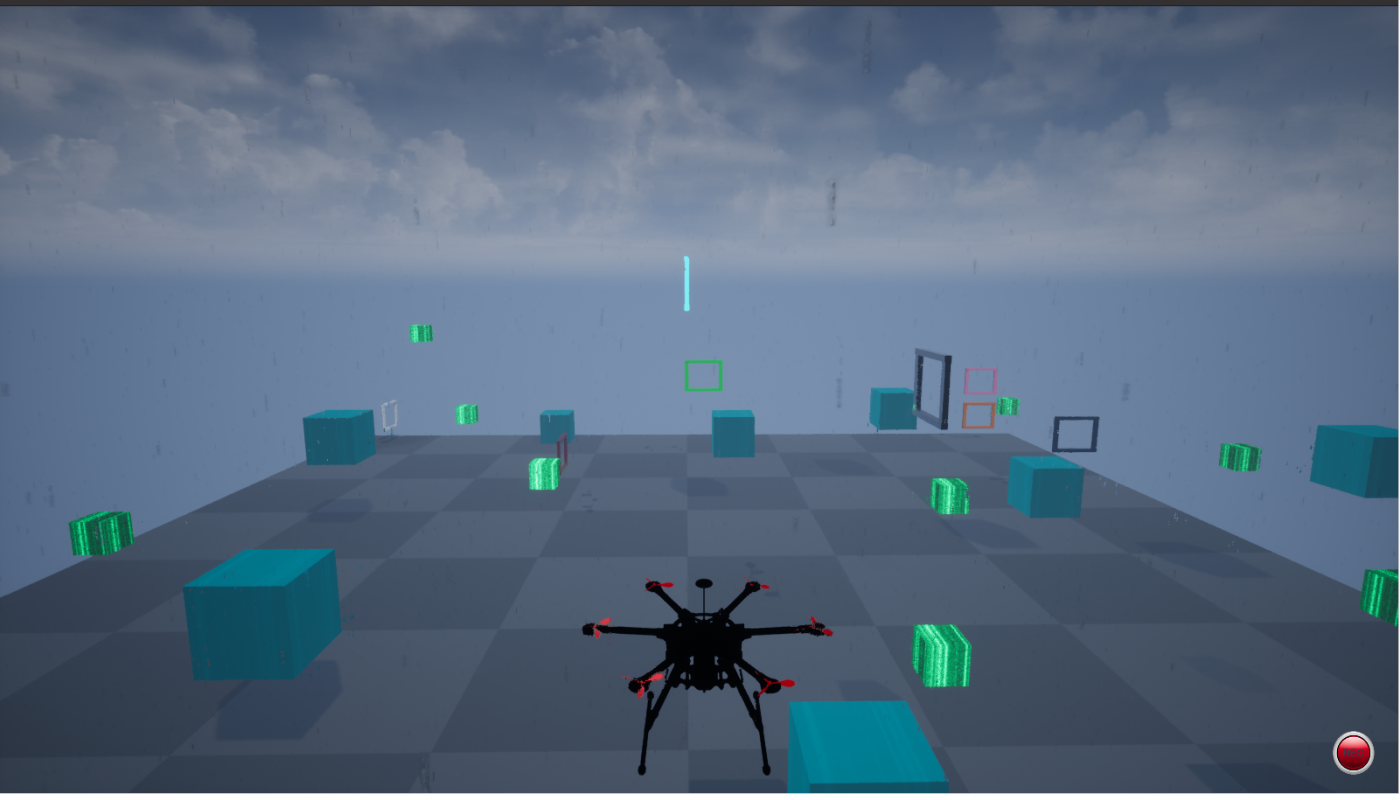}}
\end{subfigure}
\caption{\tbold{Evaluation Environments.}}
\label{fig:sim_envs}
\vspace{-5mm}
\end{figure*}

 \subsection{Analyzing Safety}
We conclude this section by formalizing the notion that BRSL enables safe RL.

\begin{theorem}\label{thm:BRSL_is_safe}
Suppose the assumptions on the robot and environment from Section \ref{sec:pb} all hold, and, at time $k = 0$, the robot is \tbold{at safe state}.
Suppose also that, at each time $k > 0$, the robot \tbold{rolls out} a new $\plan_k$, then adjusts the plan using Algorithm \ref{alg:adjust}.
Then, the robot is guaranteed to be safe at all times $k \geq 0$.
\end{theorem}
\begin{proof}
We prove the claim by induction on $k$.
At time $0$, the robot can apply $\action\brk$ to stay safe for all time.
Assume a safe plan exists at time $k \in \N$.
Then, if the output of Algorithm \ref{alg:adjust} is unsafe (no new plan found), the robot can continue its previous safe plan; otherwise, if a new plan is found, the plan is safe for three reasons.
First, the black-box reachability in Algorithm \ref{alg:LipReachability} is guaranteed to contain the true reachable set of the system \cite[Theorem 2]{alanwar2020data_conf}, because process noise is bounded by a zonotope as in Assumption \ref{assumption:noise_zonotope}.
Second, when adjusting an unsafe plan with Algorithm \ref{alg:adjust}, the zonotope collision check is guaranteed to always detect collisions \cite[Prop. 2]{conf:const_zono} to assess if $\reachapprox_j \cap \statespace\obs$ is empty for each time step $j$ of the plan.
Third, Algorithm \ref{alg:adjust} requires that, after $\nplan$ timesteps, the robot is stopped, so the new plan contains a failsafe manuever, and the robot can safely apply $\action\brk$ for all time $j \geq k+\nplan$.
\end{proof}

We note that the accuracy of the environment model, trained online, does not affect safety; Theorem \ref{thm:BRSL_is_safe} holds as long as the offline data are representative of the robot's dynamics at runtime.
We leave updating the data online for future work.

\begin{figure*}
\centering     

\begin{subfigure}[h][Turtlebot3 Reward]{\label{fig:turtlebot_reward}\includegraphics[scale=0.25]{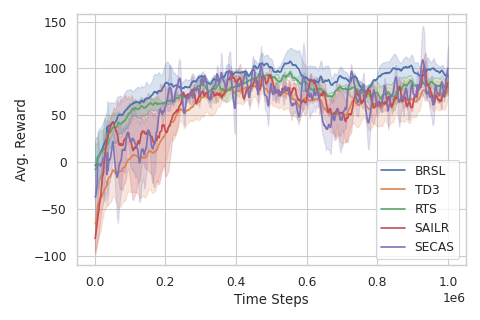}}
\end{subfigure}
\begin{subfigure}[h][Quadrotor Reward]{\label{fig:drone_reward}\includegraphics[scale=0.25]{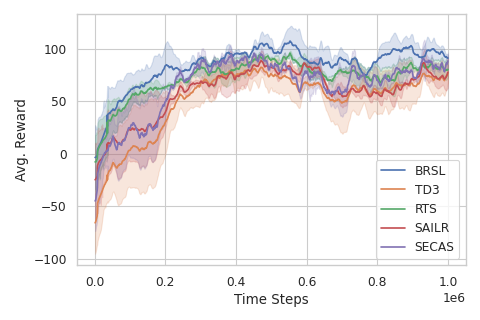}}
\end{subfigure}
\begin{subfigure}[h][Point Reward]{\label{fig:point_reward}\includegraphics[scale=0.25]{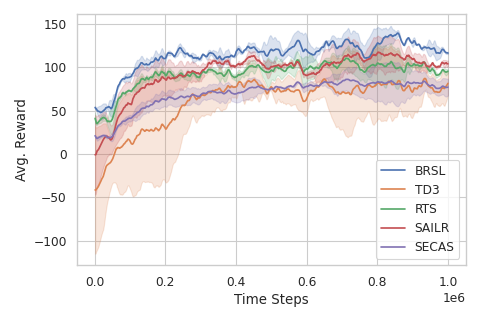}}
\end{subfigure}
\begin{subfigure}[h][\tbold{Hexarotor Reward}]{\label{fig:hexarotor_reward}\includegraphics[scale=0.25]{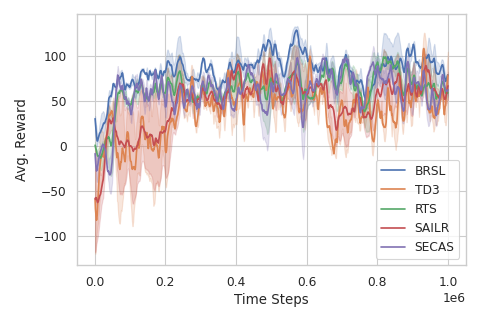}}
\end{subfigure}

\caption{Average reward over time of BRSL, RTS \cite{conf:modelBasedReachRL}, SAILR \cite{wagener2021safe}, and a vanilla TD3 baseline for each of our experiments.}
\label{fig:sim_results}
\vspace{-2mm}
\end{figure*}

\begin{table*}[ht]
\caption{Goal-based experiment results (best values in bold)}
\label{tbl:comparison}
\vspace{-3mm}
\centering
\resizebox{1.98\columnwidth}{!}{\begin{tabular}{ c|c|c|c|c|c||c|c|c|c|c } 
 & \multicolumn{5}{c||}{Turtlebot} & \multicolumn{5}{c}{Quadrotor} \\
 & BRSL & RTS & SAILR & \tbold{SECAS}&Baseline  & BRSL & RTS & SAILR & \tbold{SECAS}& Baseline  \\
 \hline
 Goal Rate [\%] &
    \textbf{57} & 52 & 48 & \tbold{53} & 42  &
    \textbf{76} & 66 & 61 & \tbold{59} & 54 \\
 Collision Rate [\%] &
    \textbf{0.0} & \textbf{0.0} & 7.3 & \tbold{\textbf{0.0}} &48 
    &\textbf{0.0} & \textbf{0.0} & 9.2& \tbold{\textbf{0.0}} & 59\\
 Mean/Max Speed [m/s] & 
    .07 / \textbf{0.18} & 0.05 / 0.15 & .07 / 0.17 & \tbold{.06 / 0.15} &\textbf{.08} / \textbf{.18} & 
    3.6 / \textbf{7.9} &  3.3 / 7.8 & \textbf{3.7} / \textbf{7.9} & \tbold{3.2 / 7.8} & \textbf{3.7} / \textbf{7.9}\\
 Mean Reward &
    \textbf{86} & 78 & 68 & \tbold{66} & 63 &
    \textbf{82} & 73 & 61 & \tbold{68} & 57 \\
 Mean $\pm$ Std. Dev.  &
    \multirow{2}{*}{50.41 $\pm$ 20.5} & \multirow{2}{*}{100.74 $\pm$ 60.5} & \multirow{2}{*}{30.53 $\pm$ 10.8} & \tbold{\multirow{2}{*}{22.4 $\pm$ 14.66}} & \multirow{2}{*}{\textbf{10.21} $\pm$ 10.05} &
    \multirow{2}{*}{60.33 $\pm$ 20.34} & \multirow{2}{*}{260.85 $\pm$ 140.67} & \multirow{2}{*}{45.31 $\pm$ 20.84} & \tbold{\multirow{2}{*}{39.7 $\pm$ 34.08}} & \multirow{2}{*}{\textbf{20.63} $\pm$ 30.2} \\
    Compute Time [ms] &&&&&&& &&\\
 \hline
    \end{tabular}}
\vspace{-2mm}
\end{table*}

\begin{table*}[ht]
    \caption{Path Following Results (best values in bold)}
    \label{tbl:point_robot}
    \vspace{-2mm}
    \centering
    \resizebox{1.98\columnwidth}{!}{\begin{tabular}{ c|c|c|c|c|c||c|c|c|c|c|c}
    & \multicolumn{5}{c||}{Point Environment} & \multicolumn{5}{c}{\tbold{Hexarotor}} \\
         & BRSL & RTS & SAILR &\tbold{SECAS} &Baseline & \tbold{BRSL} & \tbold{RTS} & \tbold{SAILR} &\tbold{SECAS} &\tbold{Baseline}\\
         \hline
         Collisions [\%] &
            \textbf{0.0} & \textbf{0.0} & 4.9 & \tbold{\textbf{0.0}} &11.4 & \tbold{\textbf{0.0}} & \tbold{2} & \tbold{14} & \tbold{7} & \tbold{63} \\
         Mean Speed [m/s] & 
            0.76 & 0.72 & 0.78 & \tbold{0.68} &\textbf{0.86} & \tbold{3.81} & \tbold{3.4} & \tbold{3.84} & \tbold{3.1} & \tbold{\textbf{3.94}} \\
         Max Speed [m/s] & 
            \textbf{2.00} & 1.89 & \textbf{2.00} & \tbold{1.74} &\textbf{2.00} & \tbold{8.4} & \tbold{8.1} & \tbold{8.2} & \tbold{7.95} & \tbold{\textbf{8.7}} \\
         Mean Reward &
            \textbf{118} & 93 & 88 & \tbold{78} &73 & \tbold{\textbf{84}} & \tbold{78} & \tbold{69} & \tbold{62} & \tbold{57}\\
         Compute Time [ms] &
            30.0 $\pm$ 10.2 & 60.18 $\pm$ 20.09 & 20.49 $\pm$ 10.34 & \tbold{16.42 $\pm$ 8.7} &\textbf{8.64} $\pm$ 1.14 & \tbold{71.93 $\pm$ 34.52} & \tbold{290.96 $\pm$ 157.24} & \tbold{67.86 $\pm$ 22.34} & \tbold{48.53 $\pm$ 28.69} & \tbold{\textbf{32.79 $\pm$ 27.6}} \\
         \hline
    \end{tabular}}
\vspace{-2mm}
\end{table*}

\section{Evaluation} \label{sec:eval}

We evaluate BRSL on two types of environments: safe navigation to a goal (a Turtlebot in Gazebo and a quadrotor platform in Unreal Engine $4$), and path following (a point mass based on \cite{achiam2017constrained,wagener2021safe} and a hexarotor in wind in Unreal Engine $4$).
Figure \ref{fig:sim_envs} shows example environments.
All code is run on a desktop computer with an Intel i5 11600 CPU and a RTX 3060 GPU. 
Our code is available \href{https://github.com/Mahmoud-Selim/Safe-Reinforcement-Learning-for-Black-Box-Systems-Using-Reachability-Analysis}{\underline{\color{blue}{online}}}.
We aim to assess the following:
(a) How does BRSL compare against other safe RL methods (RTS \cite{conf:modelBasedReachRL}, SAILR \cite{wagener2021safe}, and SECAS \cite{dalal2018safe})?
(b) How conservative is BRSL?
(c) Can BRSL run in real time?

\textbf{Setup.}
We use TD3 \cite{conf:TD3} as our RL agent, after determining empirically that it outperforms SAC \cite{conf:SAC} and DDPG \cite{conf:DDPG}. 
We randomly initialize the policy.
Since the agent outputs continuous actions, to aid the exploration process, we inject zero-mean Gaussian noise with a variance of $0.5$ that is dampened each time step.
Note that this does not affect safety since our safety layer adjusts the output of the RL agent.

For each robot, to perform reachability analysis with Algorithm \ref{alg:LipReachability}, we collect $500$ time steps of noisy state/input data (as per \eqref{eq:data}) offline in an empty environment while applying random control inputs.
We found this quantity of data sufficient to ensure safety empirically; we leave a formal analysis of the minimum amount of data for future work.

We parameterize the environment \tbold{model} as an ensemble of neural networks, each modeling a Gaussian distribution over future states and observations.
Each network has $4$ layers, with hidden layers of size 200, and leaky ReLU activations with a negative slope of $0.01$.
We use a stochastic model wherein the ensemble predicts the parameters of a probability distribution, which is sampled to produce a state as in \cite{chua2018deep}.

\textit{Goal-Based Environments.}
The Turtlebot 3 and the quadrotor seek to navigate to a random circular goal region $\statespace\goal \subset \statespace$ while avoiding randomly-generated obstacles $\statespace\obs \subset \statespace$. 
Each robot starts in a safe location at the center of the map.
Each task is episodic, ending if the robot reaches the goal, crashes, or exceeds a time limit.
Both robots have uncertain, noisy dynamics as in \eqref{eq:sys}.
We discretize time at $10$ Hz.

The Turtlebot's control inputs are longitudinal velocity in $[0.00,0.25]$ m/s and angular velocity in $[-0.5, 0.5]$ rad/s (these are the bounds of $\actionspace_k$).
The robot has wheel encoders, plus a planar lidar that generates $18$ range measurements evenly spaced in a $180^\circ$ arc in front of the robot.
The robot requires $\nbrk = 6$ time steps to stop, so we set $\nplan = 8$.

The quadrotor control inputs are commanded velocities up to 5 m/s in each spatial direction at each time step.
\tb{We note that we also experimented with learning low-level rotor speeds versus high-level velocity commands, and found that the velocity commands created the fairest testing conditions across all agents}. The robot is equipped with an IMU and a 16-channel lidar which receives range measurements around the robot in a $50^\circ$ vertical arc and a $360^\circ$ horizontal arc.
The robot has $\nbrk = 10$, so we set $\nplan = 11$.

\textit{Path Following Environments.}
These experiments assess BRSL's conservativeness is by placing the highest reward adjacent to obstacles.

The goal for the point robot is to follow a circular path of radius $r$ as quickly as possible while constrained to a region smaller than the target circle.
The point robot is a 2-D double integrator with position and velocity as its state: $\vcstate_k = (x_k,y_k,\dot{x}_k,\dot{y}_k)$.
It has a maximum velocity of 2 m/s, and its control input is acceleration up to 1 m/s$^2$ in any direction.
We use these dynamics as in \cite{achiam2017constrained,wagener2021safe} to enable a fair test against other methods that require a robot model.
We define a box-shaped safe set (the complement of the obstacle set) as $\statespace\safe = \{\vcstate_k \in \statespace : |x_k| \le x_{\max},  |y_k| \le y_{\max}\}$, with $\norm{(x_{\max}, y_{\max})}_2 < r$.
We use a reward that encourages traveling quickly near the unsafe set: $\rewfunc (\RLstate_k,\RLaction_k) = \frac{(\dot{x}_k, \dot{y}_k) \cdot (-y_k, x_k)}{1 + \big| \norm{(x_k, y_k)}_2 - r\big|}$.

The hexarotor has the same setup as the quadrotor, but with the addition of wind as an external disturbance.
The goal of the hexarotor is to pass through $10$ checkpoints in a fixed order while subject to wind (constant speed and direction) and randomly-placed obstacles.
Note, offline data collection was performed under wind conditions.
\textbf{Results and Discussion.}
\tb{The results are summarized in Tables \ref{tbl:comparison} and \ref{tbl:point_robot}, and in Figure \ref{fig:sim_results}.
BRSL outperforms the other methods in terms of reward and safety, is not overly conservative, and can operate in real time, despite lacking a model of the robot \emph{a priori}.
While SAILR and the baseline RL agent achieved higher speeds, both experienced collisions, unlike BRSL, RTS, and SECAS.
In contrast to RTS, which chooses from a low-dimensional parameterized plans, BRSL outputs a more flexible sequence of actions.
Furthermore RTS' planning time increases with state space dimension due to computing a halfspace representation of reachable set zonotopes, which grows exponentially in the number of generators \cite{althoffphdthesis}.
BRSL avoids this computation by using \eqref{prog:collision_check}.
Instead, the quantity of data for BRSL determines the computation time of $\vc{M}_j$ from Algorithm \ref{alg:LipReachability}, used for reachability and adjusting unsafe actions.
Therefore, one can ensure the amount of data allows real time operation; choosing the data optimally is left to future work.
Finally, BRSL uses zonotopes to exactly represent safety constraints, whereas SECAS uses a more conservative first-order approximation.
This results in BRSL achieving higher reward with slightly slower computation time (but still fast enough for real time operation).}


\section{Conclusion}\label{sec:con}

This paper proposes the Black-box Reachability Safety Layer, or BRSL, for safe RL without having a system model \textit{a priori}.
BRSL ensures safety via data-driven reachability analysis and a novel technique to push reachable sets out of collision.
To enable the RL agent to make dynamics-informed decisions, BRSL also learns an environment model online, which does not affect the safety guarantee.
The framework was evaluated on four robot motion planning problems, wherein BRSL respects safety constraints while achieving a high reward over time in comparison to state-of-the-art methods.
For future work, we will explore continuous-time settings, reducing the conservativeness of our reachability analysis, and minimizing the amount of data needed to guarantee safety.

\bibliographystyle{IEEEtran}
\bibliography{ref} 

\begin{thebibliography}{10}
\providecommand{\url}[1]{#1}
\csname url@samestyle\endcsname
\providecommand{\newblock}{\relax}
\providecommand{\bibinfo}[2]{#2}
\providecommand{\BIBentrySTDinterwordspacing}{\spaceskip=0pt\relax}
\providecommand{\BIBentryALTinterwordstretchfactor}{4}
\providecommand{\BIBentryALTinterwordspacing}{\spaceskip=\fontdimen2\font plus
\BIBentryALTinterwordstretchfactor\fontdimen3\font minus
  \fontdimen4\font\relax}
\providecommand{\BIBforeignlanguage}[2]{{%
\expandafter\ifx\csname l@#1\endcsname\relax
\typeout{** WARNING: IEEEtran.bst: No hyphenation pattern has been}%
\typeout{** loaded for the language `#1'. Using the pattern for}%
\typeout{** the default language instead.}%
\else
\language=\csname l@#1\endcsname
\fi
#2}}
\providecommand{\BIBdecl}{\relax}
\BIBdecl

\bibitem{book:sutton1998introduction}
R.~S. Sutton, A.~G. Barto \emph{et~al.}, \emph{Introduction to reinforcement
  learning}.\hskip 1em plus 0.5em minus 0.4em\relax MIT press Cambridge, 1998,
  vol. 135.

\bibitem{mihatsch2002risk}
O.~Mihatsch and R.~Neuneier, ``Risk-sensitive reinforcement learning,''
  \emph{Machine learning}, vol.~49, no.~2, pp. 267--290, 2002.

\bibitem{garcia2015comprehensive}
J.~Garc{\i}a and F.~Fern{\'a}ndez, ``A comprehensive survey on safe
  reinforcement learning,'' \emph{Journal of Machine Learning Research},
  vol.~16, no.~1, pp. 1437--1480, 2015.

\bibitem{geibel2005risk}
P.~Geibel and F.~Wysotzki, ``Risk-sensitive reinforcement learning applied to
  control under constraints,'' \emph{Journal of Artificial Intelligence
  Research}, vol.~24, pp. 81--108, 2005.

\bibitem{moldovan2012safe}
T.~M. Moldovan and P.~Abbeel, ``Safe exploration in markov decision
  processes,'' 2012.

\bibitem{pmlr-v37-sui15}
\BIBentryALTinterwordspacing
Y.~Sui, A.~Gotovos, J.~Burdick, and A.~Krause, ``Safe exploration for
  optimization with gaussian processes,'' in \emph{Proceedings of the 32nd
  International Conference on Machine Learning}, ser. Proceedings of Machine
  Learning Research, F.~Bach and D.~Blei, Eds., vol.~37.\hskip 1em plus 0.5em
  minus 0.4em\relax Lille, France: PMLR, 07--09 Jul 2015, pp. 997--1005.
  [Online]. Available: \url{https://proceedings.mlr.press/v37/sui15.html}
\BIBentrySTDinterwordspacing

\bibitem{altman1999constrained}
E.~Altman, \emph{Constrained Markov decision processes: stochastic
  modeling}.\hskip 1em plus 0.5em minus 0.4em\relax Routledge, 1999.

\bibitem{BORKAR2005207}
\BIBentryALTinterwordspacing
V.~Borkar, ``An actor-critic algorithm for constrained markov decision
  processes,'' \emph{Systems \& Control Letters}, vol.~54, no.~3, pp. 207--213,
  2005. [Online]. Available:
  \url{https://www.sciencedirect.com/science/article/pii/S0167691104001276}
\BIBentrySTDinterwordspacing

\bibitem{chow2017riskconstrained}
Y.~Chow, M.~Ghavamzadeh, L.~Janson, and M.~Pavone, ``Risk-constrained
  reinforcement learning with percentile risk criteria,'' 2017.

\bibitem{tessler2018reward}
C.~Tessler, D.~J. Mankowitz, and S.~Mannor, ``Reward constrained policy
  optimization,'' 2018.

\bibitem{bohez2019value}
S.~Bohez, A.~Abdolmaleki, M.~Neunert, J.~Buchli, N.~Heess, and R.~Hadsell,
  ``Value constrained model-free continuous control,'' 2019.

\bibitem{ding2020provably}
D.~Ding, X.~Wei, Z.~Yang, Z.~Wang, and M.~R. Jovanović, ``Provably efficient
  safe exploration via primal-dual policy optimization,'' 2020.

\bibitem{efroni2020explorationexploitation}
Y.~Efroni, S.~Mannor, and M.~Pirotta, ``Exploration-exploitation in constrained
  mdps,'' 2020.

\bibitem{achiam2017constrained}
J.~Achiam, D.~Held, A.~Tamar, and P.~Abbeel, ``Constrained policy
  optimization,'' 2017.

\bibitem{bharadhwaj2021conservative}
H.~Bharadhwaj, A.~Kumar, N.~Rhinehart, S.~Levine, F.~Shkurti, and A.~Garg,
  ``Conservative safety critics for exploration,'' 2021.

\bibitem{riskmetricsRL}
C.~Gehring and D.~Precup, ``Smart exploration in reinforcement learning using
  absolute temporal difference errors,'' in \emph{Proceedings of the 2013
  international conference on Autonomous agents and multi-agent systems}, 2013,
  pp. 1037--1044.

\bibitem{neuroevolutionaryinproceedings}
R.~Koppejan and S.~Whiteson, ``Neuroevolutionary reinforcement learning for
  generalized helicopter control,'' in \emph{Proceedings of the 11th Annual
  conference on Genetic and evolutionary computation}, 2009, pp. 145--152.

\bibitem{abbeel2005exploration}
P.~Abbeel and A.~Y. Ng, ``Exploration and apprenticeship learning in
  reinforcement learning,'' in \emph{Proceedings of the 22nd international
  conference on Machine learning}, 2005, pp. 1--8.

\bibitem{molnar2021model}
T.~G. Molnar, R.~K. Cosner, A.~W. Singletary, W.~Ubellacker, and A.~D. Ames,
  ``Model-free safety-critical control for robotic systems,'' \emph{IEEE
  Robotics and Automation Letters}, vol.~7, no.~2, pp. 944--951, 2021.

\bibitem{liu2022robot}
P.~Liu, D.~Tateo, H.~B. Ammar, and J.~Peters, ``Robot reinforcement learning on
  the constraint manifold,'' in \emph{Conference on Robot Learning}.\hskip 1em
  plus 0.5em minus 0.4em\relax PMLR, 2022, pp. 1357--1366.

\bibitem{chow2018lyapunov}
Y.~Chow, O.~Nachum, E.~Duenez-Guzman, and M.~Ghavamzadeh, ``A lyapunov-based
  approach to safe reinforcement learning,'' \emph{Advances in neural
  information processing systems}, vol.~31, 2018.

\bibitem{cheng2019end}
R.~Cheng, G.~Orosz, R.~M. Murray, and J.~W. Burdick, ``End-to-end safe
  reinforcement learning through barrier functions for safety-critical
  continuous control tasks,'' in \emph{Proceedings of the AAAI Conference on
  Artificial Intelligence}, vol.~33, no.~01, 2019, pp. 3387--3395.

\bibitem{modelfreecbfsIVM}
\BIBentryALTinterwordspacing
E.~Squires, R.~Konda, S.~Coogan, and M.~Egerstedt, ``Model free barrier
  functions via implicit evading maneuvers,'' 2021. [Online]. Available:
  \url{https://arxiv.org/abs/2107.12871}
\BIBentrySTDinterwordspacing

\bibitem{wagener2021safe}
N.~Wagener, B.~Boots, and C.-A. Cheng, ``Safe reinforcement learning using
  advantage-based intervention,'' \emph{arXiv preprint arXiv:2106.09110}, 2021.

\bibitem{robustCBFarticle}
M.~Jankovic, ``Robust control barrier functions for constrained stabilization
  of nonlinear systems,'' \emph{Automatica}, vol.~96, pp. 359--367, 10 2018.

\bibitem{cbfforunmodeleddynamics}
\BIBentryALTinterwordspacing
P.~Seiler, M.~Jankovic, and E.~Hellstrom, ``Control barrier functions with
  unmodeled dynamics using integral quadratic constraints,'' 2021. [Online].
  Available: \url{https://arxiv.org/abs/2108.10491}
\BIBentrySTDinterwordspacing

\bibitem{saferlforchanginglanes}
H.~Krasowski, X.~Wang, and M.~Althoff, ``Safe reinforcement learning for
  autonomous lane changing using set-based prediction,'' in \emph{2020 IEEE
  23rd International Conference on Intelligent Transportation Systems
  (ITSC)}.\hskip 1em plus 0.5em minus 0.4em\relax IEEE, 2020, pp. 1--7.

\bibitem{conf:modelBasedReachRL}
Y.~S. Shao, C.~Chen, S.~Kousik, and R.~Vasudevan, ``Reachability-based
  trajectory safeguard (rts): A safe and fast reinforcement learning safety
  layer for continuous control,'' \emph{IEEE Robotics and Automation Letters},
  vol.~6, no.~2, pp. 3663--3670, 2021.

\bibitem{LewEtAl2021}
T.~Lew, A.~Sharma, J.~Harrison, A.~Bylard, and M.~Pavone, ``Safe active
  dynamics learning and control: A sequential exploration-exploitation
  framework,'' \emph{IEEE Transactions on Robotics}, 2022, in Press.

\bibitem{dalal2018safe}
G.~Dalal, K.~Dvijotham, M.~Vecerik, T.~Hester, C.~Paduraru, and Y.~Tassa,
  ``Safe exploration in continuous action spaces,'' \emph{arXiv preprint
  arXiv:1801.08757}, 2018.

\bibitem{lavalle2006planning}
S.~M. LaValle, \emph{Planning algorithms}.\hskip 1em plus 0.5em minus
  0.4em\relax Cambridge university press, 2006.

\bibitem{kousik2020bridging}
S.~Kousik, S.~Vaskov, F.~Bu, M.~Johnson-Roberson, and R.~Vasudevan, ``Bridging
  the gap between safety and real-time performance in receding-horizon
  trajectory design for mobile robots,'' \emph{The International Journal of
  Robotics Research}, vol.~39, no.~12, pp. 1419--1469, 2020.

\bibitem{leung2020infusing}
K.~Leung, E.~Schmerling, M.~Zhang, M.~Chen, J.~Talbot, J.~C. Gerdes, and
  M.~Pavone, ``On infusing reachability-based safety assurance within planning
  frameworks for human--robot vehicle interactions,'' \emph{The International
  Journal of Robotics Research}, vol.~39, no. 10-11, pp. 1326--1345, 2020.

\bibitem{chen2016decentralized}
Y.~F. Chen, M.~Liu, M.~Everett, and J.~P. How, ``Decentralized
  non-communicating multiagent collision avoidance with deep reinforcement
  learning,'' 2016.

\bibitem{chen2018socially}
Y.~F. Chen, M.~Everett, M.~Liu, and J.~P. How, ``Socially aware motion planning
  with deep reinforcement learning,'' 2018.

\bibitem{long2018optimally}
P.~Long, T.~Fan, X.~Liao, W.~Liu, H.~Zhang, and J.~Pan, ``Towards optimally
  decentralized multi-robot collision avoidance via deep reinforcement
  learning,'' 2018.

\bibitem{e2efirst}
L.~Tai, G.~Paolo, and M.~Liu, ``Virtual-to-real deep reinforcement learning:
  Continuous control of mobile robots for mapless navigation,'' in \emph{2017
  IEEE/RSJ International Conference on Intelligent Robots and Systems
  (IROS)}.\hskip 1em plus 0.5em minus 0.4em\relax IEEE, 2017, pp. 31--36.

\bibitem{tai2018socially}
L.~Tai, J.~Zhang, M.~Liu, and W.~Burgard, ``Socially compliant navigation
  through raw depth inputs with generative adversarial imitation learning,''
  2018.

\bibitem{conf:reach1984}
Y.~Ohta, H.~Maeda, and S.~Kodama, ``Reachability, observability, and
  realizability of continuous-time positive systems,'' \emph{SIAM journal on
  control and optimization}, vol.~22, no.~2, pp. 171--180, 1984.

\bibitem{althoffphdthesis}
M.~Althoff, ``Reachability analysis and its application to the safety
  assessment of autonomous cars,'' Ph.D. dissertation, Technische
  Universit{\"a}t M{\"u}nchen, 07 2010.

\bibitem{koller2018learning}
T.~Koller, F.~Berkenkamp, M.~Turchetta, and A.~Krause, ``Learning-based model
  predictive control for safe exploration,'' in \emph{2018 IEEE conference on
  decision and control (CDC)}.\hskip 1em plus 0.5em minus 0.4em\relax IEEE,
  2018, pp. 6059--6066.

\bibitem{alanwar2020data_conf}
A.~Alanwar, A.~Koch, F.~Allg{\"o}wer, and K.~H. Johansson, ``Data-driven
  reachability analysis using matrix zonotopes,'' in \emph{Learning for
  Dynamics and Control}.\hskip 1em plus 0.5em minus 0.4em\relax PMLR, 2021, pp.
  163--175.

\bibitem{conf:const_zono}
J.~K. Scott, D.~M. Raimondo, G.~R. Marseglia, and R.~D. Braatz, ``Constrained
  zonotopes: A new tool for set-based estimation and fault detection,''
  \emph{Automatica}, vol.~69, pp. 126--136, 2016.

\bibitem{raghuraman2020set_ops_conzono}
V.~Raghuraman and J.~P. Koeln, ``Set operations and order reductions for
  constrained zonotopes,'' \emph{arXiv preprint arXiv:2009.06039}, 2020.

\bibitem{magdici2016fail}
S.~Magdici and M.~Althoff, ``Fail-safe motion planning of autonomous
  vehicles,'' in \emph{2016 IEEE 19th International Conference on Intelligent
  Transportation Systems (ITSC)}.\hskip 1em plus 0.5em minus 0.4em\relax IEEE,
  2016, pp. 452--458.

\bibitem{fridovich2019safely}
D.~Fridovich-Keil, J.~F. Fisac, and C.~J. Tomlin, ``Safely probabilistically
  complete real-time planning and exploration in unknown environments,'' in
  \emph{2019 International Conference on Robotics and Automation (ICRA)}.\hskip
  1em plus 0.5em minus 0.4em\relax IEEE, 2019, pp. 7470--7476.

\bibitem{althoff2015online}
D.~Althoff, M.~Althoff, and S.~Scherer, ``Online safety verification of
  trajectories for unmanned flight with offline computed robust invariant
  sets,'' in \emph{2015 IEEE/RSJ International Conference on Intelligent Robots
  and Systems (IROS)}.\hskip 1em plus 0.5em minus 0.4em\relax IEEE, 2015, pp.
  3470--3477.

\bibitem{shetty2020predicting}
A.~Shetty and G.~X. Gao, ``Predicting state uncertainty bounds using non-linear
  stochastic reachability analysis for urban gnss-based uas navigation,''
  \emph{IEEE Transactions on Intelligent Transportation Systems}, 2020.

\bibitem{vaskov2019not}
S.~Vaskov, H.~Larson, S.~Kousik, M.~Johnson-Roberson, and R.~Vasudevan,
  ``Not-at-fault driving in traffic: A reachability-based approach,'' in
  \emph{2019 IEEE Intelligent Transportation Systems Conference (ITSC)}.\hskip
  1em plus 0.5em minus 0.4em\relax IEEE, 2019, pp. 2785--2790.

\bibitem{amos2017optnet}
B.~Amos and J.~Z. Kolter, ``Optnet: Differentiable optimization as a layer in
  neural networks,'' in \emph{International Conference on Machine
  Learning}.\hskip 1em plus 0.5em minus 0.4em\relax PMLR, 2017, pp. 136--145.

\bibitem{chung2021constrained}
L.~K. Chung, A.~Dai, D.~Knowles, S.~Kousik, and G.~X. Gao, ``Constrained
  feedforward neural network training via reachability analysis,'' 2021.

\bibitem{conf:TD3}
S.~Dankwa and W.~Zheng, ``Twin-delayed ddpg: A deep reinforcement learning
  technique to model a continuous movement of an intelligent robot agent,'' in
  \emph{Proceedings of the 3rd International Conference on Vision, Image and
  Signal Processing}, 2019, pp. 1--5.

\bibitem{conf:SAC}
T.~Haarnoja, A.~Zhou, P.~Abbeel, and S.~Levine, ``Soft actor-critic: Off-policy
  maximum entropy deep reinforcement learning with a stochastic actor,'' 2018.

\bibitem{conf:DDPG}
T.~P. Lillicrap, J.~J. Hunt, A.~Pritzel, N.~Heess, T.~Erez, Y.~Tassa,
  D.~Silver, and D.~Wierstra, ``Continuous control with deep reinforcement
  learning,'' 2019.

\bibitem{chua2018deep}
K.~Chua, R.~Calandra, R.~McAllister, and S.~Levine, ``Deep reinforcement
  learning in a handful of trials using probabilistic dynamics models,''
  \emph{Advances in neural information processing systems}, vol.~31, 2018.

\end{thebibliography}


\end{document}